\documentclass[12pt,a4paper]{article}

\usepackage[hidelinks]{hyperref}
\usepackage[cleveref]{jmlrutils}
\usepackage{fancyhdr}
\usepackage{tikz}
\usepackage{pgfplots}
\pgfplotsset{compat=newest}
\usepackage{amartya_ltx}
\usepackage{mathtools}
\usepackage{mdframed}
\usepackage{bm}
\usepackage{nicefrac}
\usepackage[numbers,sort,compress]{natbib}
\usepackage{authblk}
\usepackage{amsmath}
\usepackage{amssymb}
\usepackage{graphicx}
\usepackage{xspace}
\usepackage{fullpage}

\newcommand{\abs}[1]{\lvert#1\rvert}

\newcommand{\Card}[1]{\left\lvert#1\right\rvert}

\newcommand{\Set}[1]{\left\{#1\right\}}

\newcommand{\Esymb}{\mathbb{E}}
\newcommand{\Psymb}{\mathbb{P}}

\DeclareMathOperator*{\E}{\Esymb}

\DeclareMathOperator*{\ProbOp}{\Psymb}
\renewcommand{\Pr}{\ProbOp}

\newcommand\bdot\bullet

\DeclareMathOperator{\val}{val}

\DeclareMathOperator{\poly}{poly}

\newcommand{\eg}{e.g.,\xspace}

\newcommand{\N}{\mathbb N}

\newcommand{\ldimop}{\mathrm{Ldim}}

\newcommand{\ldim}[1]{\ldimop\br{#1}}
\newcommand{\soa}{\operatorname{SOA}}
\newcommand{\dpsoa}{\operatorname{DP-SOA}}
\newcommand{\hyp}{\cH}

\newcommand{\point}
{\mathrm{Point}}
\renewcommand{\line}{\mathrm{Line}}

\newcommand{\alg}
{\cA}
\newcommand{\regret}
{M}
\newcommand{\inpseq}{\bm{c}}

\newcommand{\xeq}{x_{\mathrm{eq}}}
\newcommand{\xdif}{x_{\mathrm{dif}}}
\newcommand{\conc}[1]{\operatorname{#1-concentrated}}

\newcommand{\cA}{\mathcal A}
\newcommand{\cB}{\mathcal B}
\newcommand{\cC}{\mathcal C}
\newcommand{\cD}{\mathcal D}

\newcommand{\cH}{\mathcal H}

\newcommand{\cM}{\mathcal M}

\newcommand{\cO}{\mathcal O}

\newcommand{\cT}{\mathcal T}

\newcommand{\cX}{\mathcal X}

\renewcommand{\leq}{\leqslant}

\renewcommand{\geq}{\geqslant}

\let\epsilon=\varepsilon
\numberwithin{equation}{section}
\newcommand\MYcurrentlabel{xxx}
\newcommand{\MYstore}[2]{%
  \global\expandafter \def \csname MYMEMORY #1 \endcsname{#2}%
}
\newcommand{\MYload}[1]{%
  \csname MYMEMORY #1 \endcsname%
}
\newcommand{\MYnewlabel}[1]{%
  \renewcommand\MYcurrentlabel{#1}%
  \MYoldlabel{#1}%
}
\newcommand{\MYdummylabel}[1]{}
\newcommand{\torestate}[1]{%
  \let\MYoldlabel\label%
  \let\label\MYnewlabel%
  #1%
  \MYstore{\MYcurrentlabel}{#1}%
  \let\label\MYoldlabel%
}
\newcommand{\restatetheorem}[1]{%
  \let\MYoldlabel\label
  \let\label\MYdummylabel
  \begin{theorem*}[Restatement of \cref{#1}]
    \MYload{#1}
  \end{theorem*}
  \let\label\MYoldlabel
}
\newcommand{\restatelemma}[1]{%
  \let\MYoldlabel\label
  \let\label\MYdummylabel
  \begin{lemma*}[Restatement of \cref{#1}]
    \MYload{#1}
  \end{lemma*}
  \let\label\MYoldlabel
}
\newcommand{\restateprop}[1]{%
  \let\MYoldlabel\label
  \let\label\MYdummylabel
  \begin{proposition*}[Restatement of \cref{#1}]
    \MYload{#1}
  \end{proposition*}
  \let\label\MYoldlabel
}
\newcommand{\restatefact}[1]{%
  \let\MYoldlabel\label
  \let\label\MYdummylabel
  \begin{fact*}[Restatement of \cref{#1}]
    \MYload{#1}
  \end{fact*}
  \let\label\MYoldlabel
}
\newcommand{\restate}[1]{%
  \let\MYoldlabel\label
  \let\label\MYdummylabel
  \MYload{#1}
  \let\label\MYoldlabel
}

\newcommand{\e}{\epsilon}
\newcommand{\eps}{\epsilon}

\allowdisplaybreaks
\title{On the Growth of Mistakes in Differentially Private Online Learning: A Lower Bound Perspective}
\date{}

\usepackage[acronym]{glossaries}
 
\newacronym{dp}{DP}{Differential Privacy}
\newacronym{ml}{ML}{Machine Learning}
\newacronym{pac}{PAC}{Probably Approximately Correct}
\newacronym{erm}{ERM}{Empirical Risk Minimisation}
\newacronym{dpsoa}{DP-SOA}{Differentially Private Standard Optimal Algorithm}
\newacronym{soa}{SOA}{Standard Optimal Algorithm}
\newacronym{ldim}{\(\ldimop\)}{Littlestone dimension}

\author[1]{Daniil Dmitriev}
\author[1]{Krist\'of Szab\'o}
\author[2]{Amartya Sanyal}
\affil[1]{ETH Zurich}
\affil[2]{University of Copenhagen}

\begin{document}

\maketitle

\begin{abstract}%
  In this paper, we provide lower bounds for Differentially Private (DP) Online Learning algorithms. Our result shows that, for a broad class of~\(\br{\epsilon,\delta}\)-DP online algorithms, for number of rounds \(T\) such that \(\log T\leq \bigO{1 / \delta}\), the expected number of mistakes incurred by the algorithm grows as \(\Omega\br{\log T}\). This matches the upper bound obtained by~\citet{golowich2021littlestone} and is in contrast to non-private online learning where the number of mistakes is independent of \(T\). 
  To the best of our knowledge, our work is the first result towards settling lower bounds for DP--Online learning and partially addresses the open question in~\citet{dpopsanyal22a}.
\end{abstract}

\section{Introduction}\label{sec:intro}
\glsresetall
With the increasing need to protect the privacy of sensitive user data while conducting meaningful data analysis,~\Gls{dp}~\citep{dwork2006calibrating} has become a popular solution.~\Gls{dp} algorithms ensure that the impact of any single data sample on the output is limited, thus safeguarding individual privacy. Several works have obtained~\Gls{dp} learning algorithms for various learning problems in both theory and practice.

However, privacy does not come for free and often leads to a statistical (and sometimes computational) cost. The classical solution for non-private~\Gls{pac} learning~\citep{valiant1984theory} is via~\Gls{erm} that computes the best solution on the training data. Several works~\citep{bassily2014private,chaudhuri2011differentially} have shown that incorporating DP into~\Gls{erm} incurs a compulsory statistical cost that depends on the dimension of the problem. In the well-known setting of~\Gls{pac} learning with~\Gls{dp},~\citet{kasiviswanathan2011can} provided the first guarantees that all finite VC classes can be learned with a sample size that grows logarithmically in the size of the class. This line of research was advanced by subsequent works~\citep{beimel2013characterizing,feldman2014sample,beimel2014bounds}, resulting in the findings of~\citet{alon2022private} which established a surprising equivalence between non-private online learning and~Approximate~\Gls{dp}-\Gls{pac} learning.

Unlike the setting of~\Gls{pac} learning, Online learning captures a sequential game between a learner and an adversary. The adversary knows everything about the learner's algorithm except its random bits. In this work we consider a setting where, for a known hypothesis class \(\hyp\), the adversary chooses a sequence of data points \(\bc{x_1,\ldots,x_t}\) and the target hypothesis \(f^*\in\hyp\) prior to engaging with the learner. Then, the adversary reveals these data points one by one to the learner, who must offer a prediction for each. After each prediction, the adversary reveals the true label for that point. The learner's performance is evaluated by comparing the incurred mistakes against the theoretical minimum that could have been achieved by an optimal hypothesis in hindsight. Known as the \emph{realisable oblivious mistake bound} model, the seminal work of~\citet{littlestone1988learning} showed that i) the number of mistakes incurred by any learner is lower-bounded by the Littlestone dimension~(more precisely, \(\nicefrac{\ldim{\hyp}}{2}\)) of the target class \(\hyp\) and ii) there is an algorithm that makes at most~\(\ldim{\hyp}\) mistakes. This algorithm is commonly referred to as the~\Gls{soa}.

Recall that certain problem classes possess finite Vapnik-Chervonenkis (VC) dimensions but infinite Littlestone dimensions (such as the one-dimensional threshold problem). This, together with the equivalence between non-private online learning and ~\Gls{dp}-\Gls{pac} learning~\citep{alon2022private} implies that there exists a fundamental separation between DP-PAC learning and non-private PAC learning. In other words, some learning problems can be solved with vanishing error, as the amount of data increases, in~\Gls{pac} learning but will suffer unbounded error in~\Gls{dp}-\Gls{pac} learning. This implication was first proven for pure~\Gls{dp} by~\citet{feldman2014sample} and later for approximate~\Gls{dp} by~\citet{alon2019private}. With the debate on the sample complexity of approximate~\Gls{dp}-\Gls{pac} learning resolved, we next ask whether a similar gap exists between online learning with~\Gls{dp} and non-private online learning.~\citet{golowich2021littlestone} addressed this by introducing the~\Gls{dpsoa}, which suffers a mistake count, that increases logarithmically with the number of rounds \(T\) compared to a constant error rate in non-private online learning~\citep{littlestone1988learning}. This difference suggests a challenge in~\Gls{dp} online learning, where errors increase indefinitely as the game continues. The question of whether this growing error rate is an unavoidable aspect of DP-online learning was posed as an open question by~\citet{dpopsanyal22a}.

\paragraph{Main Result} In this work, we provide evidence that this additional cost is inevitable. Consider any hypothesis class \(\cH\) and for a learning algorithm \(\alg\). Let \(\bE\bs{\regret_{\alg}}\) be the expected number of mistakes incurred by \(\alg\) and let \(T\) be the total number of rounds for which the game is played.

We obtain a lower bound on \(\bE\bs{M_{\cA}}\) under some assumptions on the learning algorithm \(\cA\). Informally, we say an algorithm \(\cA\) is \(\conc{\beta}\)~(see~\Cref{defn:beta-conc} for a formal definition) if there is some output sequence that it outputs with probability at least \(1-\beta\) in response to a~\emph{non-distinguishing} input sequence. A~\emph{non-distinguishing} input sequence is a~(possibly repeated) sequence of input data points such that there exists some \(f_1,f_2\in\hyp\) which cannot be distinguished just by observing their output on the non-distinguishing input sequence. We prove a general statement for any hypothesis class in~\Cref{thm:main-finite} but show a informal corollary below.

\begin{corollary}[Informal Corollary of~\Cref{thm:main-finite}]\label{thm:informal-small}
     There exists a hypothesis class \(\cH\) with \(\ldim{\cH}=1\) (see~\cref{defn:littlestone}), such that for any \(\e, \delta > 0, T \leq \exp(1/(32\delta))\), and any online learner \(\alg\) that is \((\e, \delta)\)-\Gls{dp} and \(0.1\)-concentrated, there is an adversary, such that 
    \begin{equation}
        \E\bs{\regret_{\alg}} =\widetilde\Omega\br{\frac{\log T}{\e}},
        \end{equation}
    where \(\widetilde \Omega\) hides logarithmic factors in \(\e\). For \(T > \exp(1/(32\delta))\), \(\E\bs{\regret_{\alg}} = \widetilde\Omega\br{1 / \delta}\).
\end{corollary}

 While the above result uses a hypothesis class of Littlestone dimension one, our main result in~\Cref{sec:finite_horizon_result} also holds for any hypothesis class, even with Littlestone dimension greater than one. Utilising the \(\point_N\) hypothesis class (see~\cref{defn:point-class}) in~\Cref{thm:informal-small}, we demonstrate that the minimum number of mistakes a~\Gls{dp} online learner must make is bounded below by a term that increases logarithmically with the time horizon \(T\). This holds if the learning algorithm is concentrated at least when \(T\) is less than or equal to \(\exp\br{1 / (32\delta)}\). This contrasts with non-private online learning, where the number of mistakes does not increase with \(T\) in hypothesis classes with bounded Littlestone dimension, even if the learner is concentrated.\footnote{All deterministic algorithms are \(\conc{0}\) by definition.} Our result also shows that the analysis of the algorithm of~\citet{golowich2021littlestone}, which shows an upper bound of~\(\Omega\br{\log \br{T/\delta}}\) for~\Gls{dpsoa} is tight as long as \(T\leq\exp\br{1 / (32\delta)}\). However, as illustrated in~\Cref{fig:lower-bound}, this is not a limitation as for larger \(T\), since a simple~\emph{Name \& Shame} algorithm incurs lesser mistake than~\Gls{dpsoa} albeit at vacuous privacy levels (see discussion after~\Cref{thm:main-finite}).

In fact, the assumption of concentrated learners is not overly restrictive given that known~\Gls{dp}-online learning algorithms exhibit this property, as detailed in~\Cref{sec:example-conc}. Notably, the~\Gls{dpsoa} presented by~\citet{golowich2021littlestone} which is the sole DP online learning method known to achieve a mistake bound \(\bigO{\log\br{T}}\), is concentrated as shown in~\Cref{lem:dpsoa-conc}. This suggests that the lower bound holds for all potential~\Gls{dp} online learning algorithms.

Additionally, we extend our result to another class of~\Gls{dp} online algorithms, which we refer to as~\emph{uniform firing} algorithms, that are in essence juxtaposed to concentrated algorithms.  These algorithms initially select predictors at random until a certain
confidence criterion is met, prompting a switch to a consistent predictor—this transition, or `firing', is determined by the flip of a biased coin (with bias \(p_t\)), where the likelihood of firing increases with each mistake. However, the choice of how \(p_t\) increases and the selection of the predictor upon firing depend on the algorithm's design. For this specific type of algorithms, particularly in the context of learning the \(\point_3\) hypothesis class, Proposition~\ref{prop:firing-lb} establishes a lower bound on mistakes that also grows logarithmically with \(T\).

\Cref{sec:continual} discusses Continual Observation~\citep{dwork2010boosting,chan2011private}, another popular task within sequential~\Gls{dp}. We show that results on~\Gls{dp}~Continual Counters can be used to derive upper bounds in the online learning setting. Nonetheless, it is not clear whether lower bounds for that setting can transferred to~\Gls{dp}-Online learning. In addition, these upper bounds suffer a dependence on the hypothesis class size. 

Finally, we point out that, to the best of our knowledge we are unaware of any algorithms in the literature for pure~\Gls{dp} online learning. Our lower bound in~\Cref{thm:informal-small} immediately provides a lower bound for pure~\Gls{dp}. Similarly, DP Continual Counters provide a method for achieving upper bounds, specifically for the \(\point_N\) classes, albeit with a linear dependency on \(N\). Obtaining tight upper and lower bounds remains an interesting direction for future research.

\section{Preliminaries}\label{sec:prelim}
We provide the necessary definitions for both Online Learning and Differential Privacy. 
\subsection{Online Learning} We begin by defining the online learning game between a learner \(\alg\) and an adversary \(\cB\). Let \(T \in \N_+\) denote number of rounds and let \(\cX\) be some domain.
\begin{defn}[General game]
\label{def:general}
    Let  \(\hyp \subseteq \Set{0, 1}^{\cX}\) be a hypothesis class of functions from \(\cX\) to \(\Set{0, 1}\). The game between a learner \(\alg\) and adversary \(\cB\) is played as follows:
\begin{mdframed}[nobreak=true]
\begin{itemize}
    \item adversary \(\cB\) picks \(f^{\ast} \in \hyp\) and a sequence \(x_1, \ldots, x_T \in \cX\)
    \item \textbf{for} \(t = 1, \ldots, T\):
    \begin{itemize}
        \item learner \(\alg\) outputs a current prediction \(\hat f_t \in \Set{0, 1}^{\cX}\) (see~\cref{def:proper})
        \item \(\alg\) receives \((x_t, f^{\ast}(x_t))\)
    \end{itemize}
    \item \textbf{let} \(\regret_{\alg}  = \sum_{t = 1}^T \bI\bc{\hat f_t(x_t) \neq f^{\ast}(x_t)}\)
\end{itemize}
\end{mdframed}
\end{defn}
In this work, we study the following min-max problem:
\begin{equation}\label{eq:mistake-bound}
    \min_{\alg}\max_{\cB} \E \bs{\regret_{\alg}} = \min_{\alg} \max_{\cB} \sum_{t = 1}^T \Pr\bs{\hat f_t(x_t) \neq f^{\ast}(x_t)},
\end{equation}
where probability is taken over the randomness in \(\alg\). We refer to the random variable \(\regret_{\alg}\) as the mistake count of \(\cA\). The optimal mistake count of this game can be characterised by a combinatorial property of the hypothesis class \(\hyp\), called \emph{Littlestone dimension}, first shown by~\citet{littlestone1988learning}.

\paragraph{Littlestone dimension}

To define~\emph{Littlestone dimension}, we need the concept  of \emph{mistake tree}.  A mistake tree \(\cT\) is a complete binary tree, where  each internal node \(v\) corresponds to some \(x_v \in \cX\). Each root-to-leaf path of the tree -- denoted as \(x_1, x_2, \ldots, x_d, x_\mathrm{leaf}\) -- is associated with a label sequence \(y_1, \ldots, y_d\), 
where \(y_i = \bI\bc{x_{i+1} \text{ is the right child of }x_i}\).
 
We say that \(\cT\) is \emph{shattered} by \(\hyp\), if 
for every possible root-to-leaf labeled path \(((x_1, y_1), \ldots, (x_d, y_d))\), there exists \(f \in \hyp\), such that \(f(x_i) = y_i\) for all \(i \in [d]\). This concept leads us to the formal definition of the Littlestone dimension as follows:
\begin{defn}\label{defn:littlestone}
    Littlestone dimension \(\br{\ldim{\hyp}}\) of a hypothesis class \(\hyp\) is defined as the maximum depth of any mistake tree that can be shattered by  \(\hyp\).
\end{defn}

    \citet{littlestone1988learning} proved that for any hypothesis class \(\hyp\), there exists a deterministic learner, called Standard Optimal Algorithm~(\Gls{soa})
    such that \(\regret_{\mathrm{\Gls{soa}}} \leq \ldim{\hyp}\). Furthermore, for any learner \(\alg\) there exists a (possibly random) adversary \(\cB\), such that \(\E \bs{\regret_{\alg}} \geq \frac{\ldim{\hyp}}{2}\), where expectation is taken with respect to the randomness in \(\cB\). However, the~\Gls{soa} learner is not restricted to output a hypothesis in \(\hyp\) while making its predictions. Such learning algorithms are classified as \emph{improper learners}, which we define formally below:

\begin{defn}[Proper and Improper learner]\label{def:proper}
    A learner \(\alg\) for a hypothesis class \(\hyp\) is called \emph{proper} if its output  is restricted to belong to \(\hyp\). Any learner that is not~\emph{proper} is called~\emph{improper}.
\end{defn}

It is worth noting that most learners in online learning are improper learners though their output may only be simple mixtures of hypothesis in~\(\hyp\)~\citep{hanneke2021online}. We illustrate the importance of improper learners with a simple hypothesis class that we heavily use in the rest of this paper.
\begin{defn}[Point class]\label{defn:point-class}
    For \(N \in \N_+\), for domain \(\cX = [N] \coloneqq \Set{1, \ldots, N}\), define
    \begin{equation}
        \point_N \coloneqq \Set{f^{(i)}, i \in [N]}, \quad \text{where} \quad f^{(i)}(x) = \bI\bc{i = x}.
    \end{equation}
\end{defn}
Note that \(\ldim{\point_N} = 1\) for any \(N > 1\). A simple algorithm to learn \(\point_N\) predicts \(0\) for every input until it makes a mistake. The input \(i\) where it incurred the mistake must correspond to the true target concept \(f^{(i)}\). This algorithm is improper as no hypothesis \(f\) in \(\point_N\) predicts \(0\) universally over the whole domain.

\subsection{Differential Privacy}
In this work, our goal is to study learners that satisfy the~\Gls{dp} guarantee~\citep{dwork2006calibrating} defined formally below.

\begin{defn}[Approximate differential privacy]\label{defn:dp}
    An algorithm \(\alg\) is said to be \((\varepsilon, \delta)\)-\Gls{dp}, if for any two input sequences \(\tau = \br{\br{x_1, y_1}, \ldots, \br{x_T, y_T}}\) and \(\tau'= \br{\br{x_1', y_1'}, \ldots, \br{x_T', y_T'}}\), such that there exists \textbf{only one} \(t\) with \((x_t, y_t) \neq (x_t', y_t')\), it holds that 
    \begin{equation}
        \Pr(\alg(\tau) \in S) \leq \exp(\varepsilon) \Pr(\alg(\tau') \in S) + \delta,
    \end{equation}
    where \(S\) is any set of possible outcomes.
\end{defn}
When \(\delta = 0\) we recover the definition of \emph{pure differential privacy}, denoted by \(\e\)-\Gls{dp}. 
Note that for online learner \(\alg\), output at step \(t\) depends only on first \(t - 1\) elements of the input sequence. In the setting of offline learning, the inputs \(\tau,\tau'\) can be thought of as two datasets of length \(T\) and \(\cA\) as the learning algorithm that outputs one hypothesis \(f\)~(not necessarily in \(\hyp\)). If \(\cA\) simulatanously satisfies~\Gls{dp} and is a~\Gls{pac} learner~\citep{valiant1984theory}, it defines the setting of~\Gls{dp}-\Gls{pac} learning~\citep{kasiviswanathan2011can}. However,~the setting of~\Gls{dp}-online learning is more nuanced due to two reasons.\\

\noindent\textbf{Privacy of Prediction or Privacy of Predictor} The first complexity arises from what the privacy adversary observes when altering an input, termed as its \emph{view}. Since \(\cA\) provides an output hypothesis \(\hat{f}_t\in\bc{0,1}^{\cX}\) at every time step \(t\in [T]\) as shown in~\Cref{def:general}, the adversary's view could encompass the entire list of predictors. Our work, like~\citet{golowich2021littlestone}, focuses on this scenario, where the output set is \(S\subseteq \Set{0, 1}^{\cX \times T}\). Nevertheless, certain studies restrict the adversary's view to only the predictions, excluding the predictors themselves~\citep{beimel2013private, dwork2018privacy,naor2023private}. In this setting, it is also important to assume that the adversary only observes the predictions on the inputs that it did not change~\citep{kearns2015robust,kaplan2023black}; thus, they have \(S\subseteq\bc{0,1}^{\br{T-1}}\).

\noindent\textbf{Oblivious and Adaptive adversary} The second complexity is about whether the online adversary \(\cB\) pre-selects all input points or adaptively chooses the next point based on the learner \(\cA\)'s previous response. Although the former, known as an \emph{oblivious} adversary, seems less potent, this difference does not manifest itself in non-private learning~\citep{cesa2006prediction}. However, this distinction becomes significant in the context of DP online learning. Adaptive adversaries, by design, leverage historical data in their decision-making process.
While works like~\citet{kaplan2023black} focus on adaptive adversaries, others like~\citet{kearns2015robust} concentrate on oblivious ones, and~\citet{golowich2021littlestone} examine both. Our contribution lies in setting lower bounds against the simpler scenario of oblivious adversaries.

\section{Related work}\label{sec:related}
Understanding which hypothesis classes can be privately learned is an area of vibrant research and was started in the context of Valiant's~\Gls{pac} learning model~\citep{valiant1984theory}. A hypothesis class \(\cH\) is considered~\Gls{pac}-learnable if there exists an algorithm \(\alg\), which can utilize a polynomial-sized\footnote{The term `polynomial-sized' refers to a sample size that is polynomial in the PAC parameters, including the error rate, confidence level, size of the hypothesis class, and the dimensionality of the input space.}, independent, and identically distributed (i.i.d.) sample \(D\) from any data distribution to produce a hypothesis \(h\in\cH\) that achieves a low classification error with high probability on that distribution. In the context of DP-PAC learning, as defined by ~\citet{kasiviswanathan2011can}, the learner \(\alg\) must also satisfy~\Gls{dp} constraint with respect to the sample \(D\). The overarching objective in this research domain is to find a clear criterion for the private learnability of hypothesis classes, analogous to the way learnability has been characterized in non-private settings—through the Vapnik-Chervonenkis (VC) dimension for offline learning~\citep{blumer1989learnability} and the Littlestone dimension for online learning~\citep{littlestone1988learning,ben2009agnostic}.

\citet{kasiviswanathan2011can} started this line of research by showing that the sample complexity of~\Gls{dp}-\Gls{pac} learning a hypothesis class  \(\cH\) is~\(\bigO{\log\br{\abs{\cH}
}}\).~\citet{beimel2014bounds} showed that the VC dimension does not dictate the sample complexity for proper pure~\Gls{dp}-\Gls{pac} learning of the \(\point_N\) class. However, they showed that if the setting is relaxed to improper learning then this sample complexity can be improved, thus showing a separation between proper and improper learning, something that is absent in the non-private~\Gls{pac} model.~\citet{beimel2013characterizing} sharpened this result by constructing a new complexity measure called~\emph{probabilistic representation dimension} and proving that this measure characterises improper pure~\Gls{dp} exactly.

 By leveraging advanced tools from communication complexity theory, they refined the understanding of the probabilistic representation dimension and demonstrated that the sample complexity for learning a notably simple hypothesis class, denoted as \(\line_p\), under approximate improper DP-PAC conditions, is significantly lower than the corresponding lower bound established for pure contexts.

Relaxing the notion of pure~\Gls{dp} to approximate~\Gls{dp},~\citet{beimel2013private} showed that the sample complexity for proper approximate~\Gls{dp}-\Gls{pac} learning can be significantly lower than proper~pure~\Gls{dp}-\Gls{pac} learning, thereby showing a separation between pure and approximate~\Gls{dp} in the context of proper~\Gls{dp}-\Gls{pac} learning. The inquiry into whether a similar discrepancy exists in improper DP-PAC learning was resolved by~\citet{feldman2014sample} who proved a separation between pure and approximate~\Gls{dp} in the improper~\Gls{dp}-\Gls{pac} learning model. To do this, they first proved a sharper characterisation of the probabilistic representation dimension using concepts from communication complexity. Then, they showed that the sample complexity for learning a notably simple hypothesis class, denoted as \(\line_p\), under approximate improper~\Gls{dp}-\Gls{pac} conditions, is significantly lower than the corresponding lower bound established for pure~\Gls{dp}.

 ~\citet{feldman2014sample} were also the first to obtain lower bounds for~\Gls{dp}-\Gls{pac} learning that grows as \(\Omega\br{\ldim{\cH}}\), albeit limited to the pure~\Gls{dp} setting.~\citet{alon2019private} showed that it is possible to obtain a lower bound for approximate~\Gls{dp} that grows as \(\Omega\br{\log^*\br{\ldim{\cH}}}\) thus marking a clear distinction between non-private and approximate \Gls{dp}-\Gls{pac} learning. This finding illustrated that DP-PAC learning's complexity could align with that of online learning, which is similarly governed by the Littlestone dimension. In a series of subsequent works, see~\cite{alon2022private, ghazi2021sample}, a surprising connection was established between private offline learning and non-private online learning. In particular, classes that are privately offline learnable are precisely those with finite Littlestone dimension.

This naturally highlights a similar question of private online learning, in particular whether~\Gls{dp} further limits which classes are learnable in the~\Gls{dp}-Online learning model.~\citet{golowich2021littlestone} provided an algorithm, called~\Gls{dpsoa}, which has expected number of mistakes growing as~\(O(2^{2^{\ldimop}} \log T)\). Interestingly, unlike~\Gls{soa} in the non-private online setting,~\Gls{dpsoa}'s mistake count increases with number of steps \(T\) in the online game. When considering adaptive adversaries, the upper bound on mistakes escalates to \(\bigO{\sqrt{T}}\). Under a slightly weaker definition of~\Gls{dp}, known as Challenge-DP, where the privacy adversary only sees the predictions and not the whole predictor function,~\citet{kaplan2023black} obtained an upper bound of \(\bigO{\log^2\br{T}}\) for both adaptive and oblivious adversaries. However, it is not clear from these works, whether the dependence on \(T\) is unavoidable. A related setting is that of~\emph{continual observation under~\Gls{dp}} where such a dependence is indeed unavoidable under the pure~\Gls{dp} model. However, the results from continual observation do not immediately transfer to online learning as discussed in~\Cref{sec:continual}.

\section{Lower Bound for Private Online learning under Concentration assumption}\label{sec:main}

In this section, we provide the main result of our work along with their proof. Before stating the main result in~\Cref{thm:main-finite}, we need to define the concept of~\emph{distinguishing tuple} and \(\conc{\beta}\) learners.

\begin{defn}\label{defn:dist-tuple}
    Given \(f_0, f_1 \in \hyp\) and \(\xeq, \xdif \in \cX\), we call the tuple \((f_0, f_1, \xeq, \xdif)\) \emph{distinguishing},
    if  it satisfies both \(f_0(\xeq) = f_1(\xeq)\) and \(f_0(\xeq) = f_0(\xdif) \neq f_1(\xdif)\).
\end{defn}
A distinguishing tuple means that there are two functions (\(f_0, f_1\)), and two input points (\(\xeq, \xdif\)), such that only one of these points can effectively differentiate between the two functions. The absence of a distinguishing tuple implies a restricted hypothesis class: either \(\mathcal{H}\) is a singleton (\(\Card{\hyp} = 1\)), or it contains precisely two inversely related functions ( \(\Card{\hyp} = 2\) with \(f_1 = 1 - f_0\)). In the latter case, \emph{every} input point contains information distinguishing \(f_1\) and \(f_2\).
This implies that there is no difference between input sequences, and the mistake bound will not depend on \(T\). For the purposes of our analysis, we proceed under the assumption that a distinguishing tuple always exists.

Let \((f_0, f_1, \xeq, \xdif)\) be a distinguishing tuple and suppose that adversary chooses \(f^{\ast} \in \Set{f_0, f_1}\).
Knowing only information on \(f^{\ast}(\xeq)\) does not help to tell apart \(f_0\) from \(f_1\).
Furthermore, if an algorithm is `too confident', meaning that it strongly prefers output of \(f_0\) on \(\xdif\) over output of \(f_1\), 
it will necessarily make a mistake on \(\xdif\) if \(f^{\ast} = f_1\). We will use this basic intuition to obtain our main lower bound and we will call such learners `concentrated', as defined below.

For input \(\tau\), index \(t \in [T]\) and an input point \(x \in \cX\), 
we denote \(\cA(\tau)_t \bs{x}\) to be the value of the \(t\)-th output function of \(\cA(\tau)\) evaluated at point \(x\).
\begin{defn}\label{defn:beta-conc}  
    An algorithm \(\alg\) is called \(\conc{\beta}\), 
    if there exists a distinguishing tuple \((f_0, f_1, \xeq, \xdif)\), such that
    \begin{equation}
        \Pr \bs{\forall t \in [T],\ \alg \br{\tau_0}_t\bs{\xdif} = f_0(\xdif)} \geq 1 - \beta,
    \end{equation}
    where \(\tau_0 = \br{\br{\xeq, f_0(\xeq)}, \ldots, \br{\xeq, f_0(\xeq)}}\).
\end{defn}
Note that \(\tau_0\) from~\cref{defn:beta-conc} is a `dummy', \emph{non-distinguishing} input, as it does not contain any information to distinguish \(f_0\) from \(f_1\).

\subsection{Main Result}\label{sec:finite_horizon_result}
We now show that if the learner is both differentially private and concentrated, it will necessarily suffer a large (logarithmic in \(T\)) number of mistakes  in the game of~\Cref{def:general}.
\begin{thm}\label{thm:main-finite}
Let \(\hyp\) be an arbitrary hypothesis class. Let \(\e > 0\), \(\delta \leq \e^2\) and \(T \leq \exp(1 / (32\delta))\).
If, for some \(\delta \leq \beta \leq 1/10\), \(\alg\) is a \(\beta\)-concentrated \((\e, \delta)\)-\Gls{dp} online learner of \(\hyp\), then there exists an adversary \(\cB\), such that
    \begin{equation}
        \E \bs{\regret_{\alg}} = \widetilde\Omega\br{\frac{\log T / \beta}{\e}},
    \end{equation}
    where \(\widetilde \Omega\) contains logarithmic in \(\e\) factors. For \(T > \exp(1 / (32\delta))\), \(\E \bs{\regret_{\alg}} = \widetilde\Omega\br{1/\delta}\).
\end{thm}

 Before comparing our lower bound with known upper bounds, we first discuss the condition \(T\leq \exp(1 / (32\delta))\). Our bound suggests that for sufficiently large \(T\), specifically when \(T\) exceeds \(\exp(1 / (32\delta))\), the dependency on \(T\) is in fact not needed.
This can be seen from the following simple \emph{Name \& Shame} algorithm:  initialize an empty set \(S\); at each step, apply~\Gls{soa} and output \(\mathrm{\Gls{soa}}(S)\); upon receiving a new entry \((x_t, f^\ast(x_t))\), add it to \(S\) with probability \(\delta\). Clearly, this algorithm is \((0, \delta)\)-\Gls{dp}, since for any fixed input point, it only depends on this point with probability \(\delta\).
Furthermore, each time the algorithm incurs a mistake, it adds this mistake to \(S\) with probability \(\delta\). Since the algorithm runs \Gls{soa} on \(S\), it will make on expectation at most \(\ldim{\hyp} / \delta\) mistakes.
This algorithm is, however, not `conventionally' private, since it potentially discloses a \(\delta\)-fraction of the data, namely the set \(S\).

Now, we compare result of~\cref{thm:main-finite} with known upper bounds in~\Cref{fig:lower-bound}. Recall that~\Gls{dpsoa} in \citet{golowich2021littlestone} obtained an upper bound which increases logarithmically with the time horizon \(T\).~\Cref{fig:lower-bound} shows that for \(T \leq \exp(1 / (32\delta))\), the dependency on \(T\) is necessary, thereby showing the tightness of~\Gls{dpsoa}~\citep{golowich2021littlestone}. For  larger \(T\), the aforementioned \emph{Name \& Shame} actually outperforms~\Gls{dpsoa}, indicating that, for fixed \(\e, \delta\), \emph{online algorithm always compromises privacy at very large \(T\)}. However, even for smaller \(T\), the plot shows that the number of mistake must grow with increasing \(T\) until it matches the \emph{Name \& Shame} algorithm.

\begin{figure}[t]\centering\vspace{-20pt}
\includegraphics[width=0.2\linewidth]{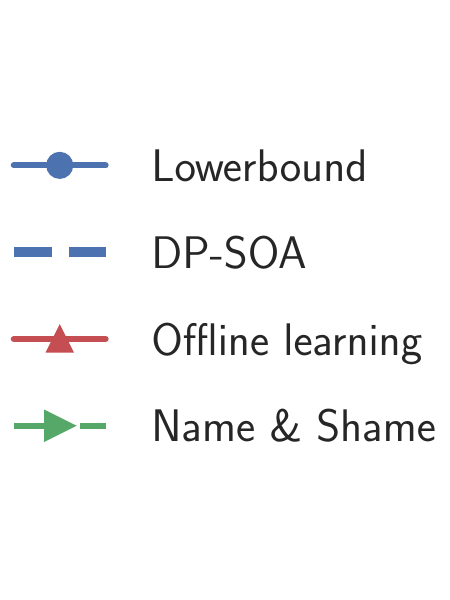}\quad\quad
\includegraphics[width=0.6\linewidth]{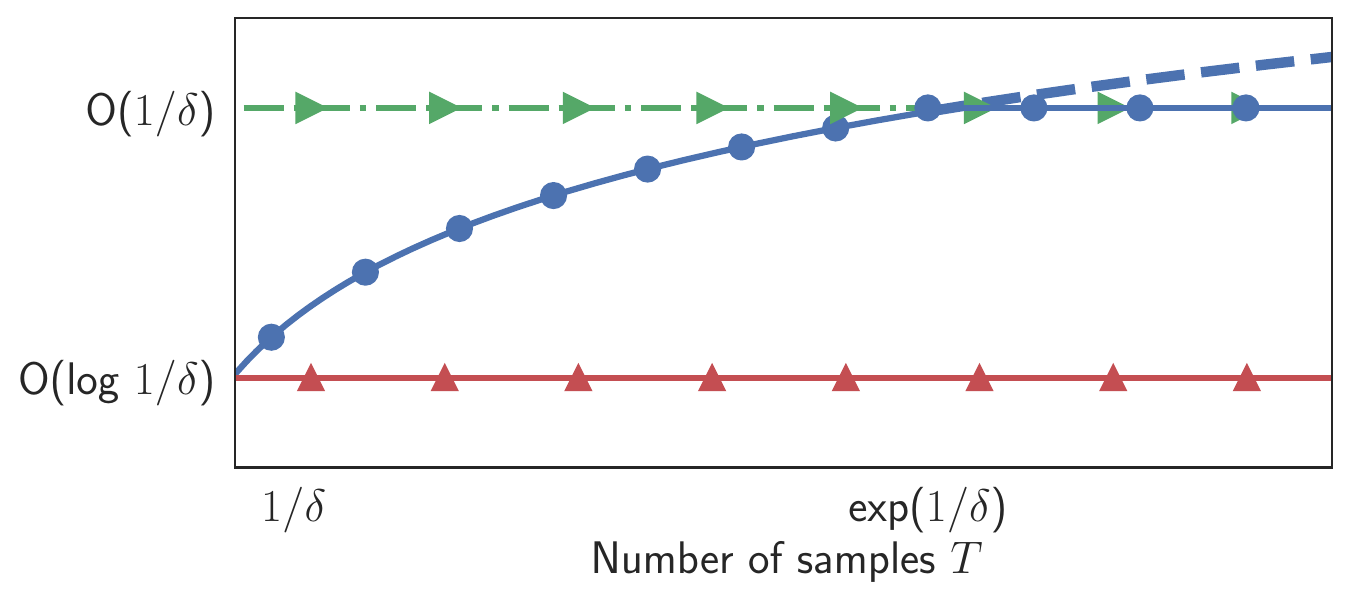}
\caption{\small Lower bound from~\cref{thm:main-finite} vs. existing upper bounds. We assume that \(\eps, \delta\) are fixed and for simplicity ignore dependence on \(\eps\). X-axis corresponds to the number of samples, growing from \(T \sim 1/\delta\) to \(T \sim \exp(1/\delta)\) and larger. Y-axis shows the expected number of mistakes, \(\E \bs{M_{\alg}}\).\vspace{-15pt}}
\label{fig:lower-bound}
\end{figure}

The main idea behind the construction of the lower bound is the following: assume that there is a distinguishing tuple \((f_0, f_1, \xeq, \xdif)\), such that \(\alg\) on \((\xeq, \ldots, \xeq)\) with high probability only outputs functions that are equal to \(f_0\) at \(\xdif\). Then, the adversary picks \(f^* = f_1\) and its goal is to select time steps \(t\) to insert \(\xdif\), such that (i) \(\alg\) with high probability will make a mistake at \(t\), and (ii) \(\alg\) will not be able to `extract a lot of information' from this mistake. As we show, both of these conditions can be guaranteed using the~\Gls{dp} property and concentration assumption of \(\alg\).
Note that if adversary inserts \(\xdif\) serially starting from \(t = 1\), \(\alg\) can have regret much smaller than \(\log T\). For example, if \(\alg\) first incurs a constant number of mistakes, the group privacy property allows a considerable change in the output distribution to only output \(f_1 = f^*\). But this is possible only if \(\alg\) can predict in advance, where it will make the mistake in the future.

Therefore, the adversary needs to `disperse' the points \(\xdif\) across \(T\) time steps, such that \(\alg\) cannot predict, where it might make the next mistake. We begin our construction by first inserting a point at the beginning. Then, depending on whether the leaner anticipates more points \(\xdif\) in the first half \emph{or} in the second half of the sequence, we insert \(\xdif\) in the half where it least expects it. Note that because of the concentration and~\Gls{dp} assumption, the learner cannot anticipate points \(\xdif\) in both halves simultaneously (recall that on input sequence consisting of only \(\xeq\), learner does not output a correct function for \(\xdif\)). By continuing this construction recursively, we are able to insert \(\Omega(\log T)\) points \(\xdif\), such that with constant probability, on each of them \(\alg\) will make a mistake.

\begin{proof}[of~\cref{thm:main-finite}]
    Since \(\alg\) is \(0.1\)-concentrated, there exists a distinguishing tuple \(\br{f_0, f_1, \xeq, \xdif}\), such that 
    \begin{equation}
        \Pr \bs{\forall t \in [T],\ \alg \br{\tau_0}_t\bs{\xdif} = f_0(\xdif)} \geq 0.9,
    \end{equation}
    where \(\tau_0 = \br{\br{\xeq, f_0(\xeq)}, \ldots, \br{\xeq, f_0(\xeq)}}\).
    WLOG assume that \(f_0(\xeq) = f_0(\xdif) = 0\).
    When \(T > \exp(1 / (32\delta))\), by simply only using the first \(\exp(1 / (32\delta))\) rounds, we obtain the required lower bound \(\widetilde\Omega(1 / \delta)\). In the remaining, we assume that \(T \leq \exp(1 / (32\delta))\).
    Furthermore, let \(k\) be the largest integer, such that \(2^k - 1 \leq T\). 
    Note that \(\frac{1}{2}\log{T} \leq k \leq 2\log T \leq 1/(16\delta)\). 
    For simplicity, we also assume that \(T = 2^k - 1\). We start with the case \(\e = \e_0 = \log\br{3/2}\) and pick \(f^\ast = f_1\) for the adversary.
    
    Note that to show \(\E \bs{M_\alg} = \Omega\br{\log T / \beta}\), we can construct two adversaries, first achieving~\(\E \bs{M_\alg} = \Omega\br{\log 1 / \beta}\)~(\textbf{Case I}) and the other~\(\E \bs{M_\alg} = \Omega\br{\log T}\)~(\textbf{Case II}). We now prove each of them.\\~\
    
    \noindent\textbf{Case I:}~We start with showing the adversary for the former bound, i.e., such that \(\E \bs{M_\alg} = \Omega\br{\log 1 / \beta}\).
    The concentration assumption implies that for any \(t \in [T]\), \(\Pr\br{\alg\br{\tau_0}_t\bs{\xdif} = 1} \leq \beta\).
    Therefore, if \(\tau_k\) contains \(k\) copies of the point \((\xdif, 1)\),  by applying~\Gls{dp} property of \(\alg\) \(k\) times, we obtain that for any \(t \in [T]\),
    \(\Pr\br{\alg\br{\tau_k}_t\bs{\xdif} = 1} \leq \beta \exp(k \e_0) + \delta \br{\frac{\exp\br{k\e_0} - 1}{\exp(\e_0) - 1}} \leq (k \delta + \beta)\exp(k \e_0)\).
    If \(k = \frac{1}{8 \e_0} \log (1 / \beta)\), we derive that for any \(t \in [T]\), 
    \begin{equation}
        \Pr(\alg(\tau_k)_t\bs{\xdif} = 1) \leq (k \delta + \beta) \exp(k \e_0) \leq \frac{\delta \log (1/\delta)}{8 \delta^{1/8} \e_0} \leq \frac{1}{3} \delta^{1/2} \log (1 / \delta) \leq \frac{1}{2}.
    \end{equation} This implies that for the sequence \(\tau_k\), expected number of mistakes is \(\E\bs{M_{\alg}} \geq \frac{1}{16 \e_0} \log (1 / \beta)\). 
    
    \noindent\textbf{Case II:} In the previous case, it did not matter where exactly the points \((\xdif, 1)\) are inserted. However, if we want to prove the lower bound \(\Omega(\log T)\), this is no longer true and one needs to be careful with the placement of the inserted points.

    In the following, we construct a sequence \(\inpseq = \br{x_1, \ldots, x_T}\), such that expected number of mistakes of \(\alg\) will be large. We proceed iteratively, maintaining scalar sequences \((l_i), (r_i)\), and sequences \(\inpseq^{(i)}\) such that 
    \begin{enumerate}
    \item \(\inpseq^{(i)}\) contains exactly \(i\) points \(\xdif\) on the prefix \([1, l_i - 1]\),
        \item \(p_i \coloneqq \Pr\bs{\forall t \in [1, l_i - 1], \alg(\inpseq^{(i)})_t\bs{\xdif} = 0} \geq 1/2 - 4i\delta\),
        \item \(q_i \coloneqq \Pr\bs{\forall t \in [1, l_i - 1], \alg(\inpseq^{(i)})_t\bs{\xdif} = 0 \textbf{ and } \exists t \in [l_i, r_i], \text{ s.t. } \alg(\inpseq^{(i)})_t\bs{\xdif} = 1} \leq 2\delta\).
    \end{enumerate}
    Assume that we obtain these sequences until \(i = k \geq \frac{1}{2} \log T\). Then on the event \(\Set{\forall t \in [1, l_k - 1], \alg(\inpseq^{(i)})_t\bs{\xdif} = 0}\), \(\alg\) will make \(k\) mistakes. Since \(k \leq 1/(16\delta)\), from the second property above \(p_k \geq 1/4\). This implies that for the sequence \(\inpseq^{(k)}\) we have \(\E\bs{\regret_{\alg}} \geq p_k k \geq \frac{1}{8} \log T\).
    
     We construct the sequences by induction, starting with \(\inpseq^{(0)} = \br{\xeq, \ldots, \xeq}\). For each \(i\), \(\inpseq^{(i)}\) will differ from \(\inpseq^{(i+1)}\) at exactly one point. This allows us to use \Gls{dp} property of \(\alg\) in order to compare outputs on \(\inpseq^{(i)}\) and \(\inpseq^{(i+1)}\). From \(\conc{\delta}\) assumption, we can pick \(l_0 = 1, r_0 = T\) which gives \(p_0 = 1\) and \(q_0 = 0.1\) (we interpret \(\Pr\bs{\forall t \in \emptyset \ldots} = 1\)).

    Given \(\inpseq^{(i)}\) we construct \(\inpseq^{(i+1)}\) by substituting \(l_i\)-th input point with \(\xdif
    \):
    \begin{enumerate}
        \item Let \((\inpseq^{(i+1)})_j = (\inpseq^{(i)})_j\) for all \(j \neq l_i\),
        \item Set \((\inpseq^{(i+1)})_{l_i} = \xdif\).
    \end{enumerate}
    Now we compute \(l_{i+1}, r_{i+1}\) and bound \(p_{i+1}, q_{i+1}\). To do this, we first introduce
    \begin{equation}
    \begin{aligned}
        &p'_i \coloneqq \Pr\bs{\forall t \in [1, l_i], \alg(\inpseq^{(i)})_t\bs{\xdif} = 0} \geq p_i - q_i, \\
        &q'_i \coloneqq \Pr\bs{\forall t \in [1, l_i], \alg(\inpseq^{(i)})_t\bs{\xdif} = 0 \textbf{ and } \exists t \in [l_i + 1, r_i], \text{ s.t. } \alg(\inpseq^{(i)})_t\bs{\xdif} = 1} \leq q_i,
    \end{aligned}
    \end{equation}
    which just account for the shift \(l_i \to l_i+1\). 
    Note that, since the events are nested, we can compute \(p'_i - q'_i = \Pr\bs{\forall t \in [1, r_i], \alg(\inpseq^{(i)})_t\bs{\xdif} = 0} = p_i - q_i\).
    Next, for \(m_i = (l_i + r_i) / 2\) and for any \(\bm{x} \in \Set{\xeq, \xdif}^T\), define
    \begin{equation}
    \begin{aligned}
        Q(\bm{x}) &\coloneqq \Set{\forall t \in [1, l_i], \alg(\bm{x})_t\bs{\xdif} = 0 \textbf{ and } \exists t \in [l_i + 1, r_i], \text{ s.t. } \alg(\bm{x})_t\bs{\xdif} = 1}, \\
        Q_1(\bm{x}) &\coloneqq \Set{\forall t \in [1, l_i], \alg(\bm{x})_t\bs{\xdif} = 0 \textbf{ and } \exists t \in [l_i + 1, m_i], \text{ s.t. } \alg(\bm{x})_t\bs{\xdif} = 1}, \\
        Q_2(\bm{x}) &\coloneqq \Set{\forall t \in [1, m_i], \alg(\bm{x})_t\bs{\xdif} = 0 \textbf{ and } \exists t \in [m_i + 1, r_i], \text{ s.t. } \alg(\bm{x})_t\bs{\xdif} = 1}.
    \end{aligned}
    \end{equation}
    Clearly, for any \(\bm{x}\), \(Q(\bm{x}) = Q_1(\bm{x}) \cup Q_2(\bm{x})\) with \(Q_1(\bm{x}) \cap Q_2(\bm{x}) = \emptyset\).
    Therefore,
    \begin{equation}
    \begin{aligned}
        q'_i = \Pr\bs{Q(\inpseq^{(i)})} = \Pr\bs{Q_1(\inpseq^{(i)})} + \Pr\bs{Q_2(\inpseq^{(i)})}.
    \end{aligned}
    \end{equation}
    We can use DP property of \(\alg\) when comparing outputs on \(\inpseq^{(i)}\) and \(\inpseq^{(i+1)}\) to get
    \begin{equation}\label{eq:q-i-plus-one}
    \begin{aligned}
    \min\br{\Pr\bs{Q_1(\inpseq^{(i+1)})}, \Pr\bs{Q_2(\inpseq^{(i+1)})}}
    &\leq \frac{1}{2} (\Pr\bs{Q_1(\inpseq^{(i+1)})} + \Pr\bs{Q_2(\inpseq^{(i+1)})})\\ 
    &= \frac{1}{2} \Pr\bs{Q(\inpseq^{(i+1)})} \leq \frac{1}{2} \br{\exp(\varepsilon_0) \Pr\bs{Q(\inpseq^{(i)})} + \delta} \\
    &= \exp(\varepsilon_0)q'_i /2 + \delta/2 \leq \frac{3}{4}q_i + \delta/2.
    \end{aligned}
    \end{equation}
    If \(\Pr\br{Q_1(\inpseq^{(i+1)})} \leq \Pr\br{Q_2(\inpseq^{(i+1)})}\) we set \(l_{i+1} \coloneqq l_i + 1, r_{i+1} \coloneqq m_i\), which gives \(q_{i+1} = \Pr\br{Q_1(\inpseq^{(i+1)})}\) and \(p_{i+1} = p'_i \geq p_i - q_i\) . 
    When \(\Pr\br{Q_1(\inpseq^{(i+1)})} > \Pr\br{Q_2(\inpseq^{(i+1)})}\), we set \(l_{i+1} = m_i + 1, r_{i+1} = r_i\), with \(q_{i+1} = \Pr\br{Q_2(\inpseq^{(i+1)})}\). We can bound 
    \begin{equation}
    \begin{aligned}
        p_{i+1} &= \Pr\br{\forall t \in [1, m_i], \alg(\inpseq^{(i+1)})_t\bs{\xdif} = 0} \\ &= \Pr\br{\forall t \in [1, l_i], \alg(\inpseq^{(i+1)})_t\bs{\xdif} = 0} \\
        &\quad - \Pr\br{\forall t \in [1, l_i], \alg(\inpseq^{(i+1)})_t\bs{\xdif} = 0 \textbf{ and } \exists t \in [l_i + 1, m_i], \text{ s.t. } \alg(\inpseq^{(i + 1)})_t\bs{\xdif} = 1} \\ 
        & = p'_i - \Pr\br{\forall t \in [1, l_i], \alg(\inpseq^{(i+1)})_t\bs{\xdif} = 0 \textbf{ and } \exists t \in [l_i + 1, m_i], \text{ s.t. } \alg(\inpseq^{(i + 1)})_t\bs{\xdif} = 1} 
        \\
        & \geq p'_i - 
        \frac{3}{2} \Pr\br{\forall t \in [1, l_i], \alg(\inpseq^{(i)})_t\bs{\xdif} = 0 \textbf{ and } \exists t \in [l_i + 1, m_i], \text{ s.t. } \alg(\inpseq^{(i)})_t\bs{\xdif} = 1}
        \\ &= p'_i - \frac{3}{2} q'_i = p_i - q_i -\frac{1}{2} q'_i \geq p_i - \frac{3}{2} q_i,
    \end{aligned}
    \end{equation}
    where we used the DP property for the first inequality.
    Overall, we obtain with our choice of \(l_{i+1}, r_{i+1}\):
    \begin{equation}
    \label{eq:q-i-p-i-recursion}
    q_{i+1} \leq \frac{3}{4} q_i + \delta / 2 \quad \text{and} \quad p_{i+1} \geq p_i - 2q_i.    
    \end{equation}
    By geometric sum properties, we have that \(q_i \leq \frac{1}{10} (3/4)^i + 2\delta\), and therefore, \(p_i \geq 1/2 - 4i \delta\).
    Finally, note that \(r_{i+1} - l_{i+1} = (r_i - l_i) / 2 - 1\). In the beginning we have \(r_0 - l_0 = 2^k - 2\), therefore \(r_i - l_i = 2^{k - i} - 2\). Thus, we can repeat this process \(\frac{1}{2} \log T \leq k \leq 2 \log T \leq 1/(16\delta)\) times. By construction, all inserted points \(\xdif\) for \(\inpseq^{(i)}\) lie on the prefix \([1, l_i - 1]\). In order to extend to other values of \(\e\), we either insert more points per step (if \(\e < \e_0\)), 
    or divide the segment into more parts (if \(\e > \e_0\)). We refer to~\Cref{app:all-eps} for the argument in these cases. This concludes the proof.
\end{proof}

\subsection{Examples of \(\conc{\beta}\) online learners for \(\point_N\)}
\label{sec:example-conc}
Next, we show that several learners that could be used for online learning \(\point_N\), satisfy~\cref{defn:beta-conc}. In particular, in~\Cref{lem:example-conc-learners} we prove that any improper learner that learns \(\point_N\) using a finite union of hypothesis in \(\point_N\)~(defined as Multi-point in~\Cref{defn:multi-point}) is concentrated. \Cref{lem:dpsoa-conc} implies that the only existing private online learning algorithm~\Gls{dpsoa} is concentrated.

\begin{defn}[Multi-Point class]\label{defn:multi-point}
     For \(1 \leq K \leq N \in \N_+\), we define \(
        \point^K_N\)
    to be the class of functions that are equal to 1 on at most \(K\) points.
\end{defn}
\begin{lem}\label{lem:example-conc-learners}
    Let \(\beta > 0\) and \(K, T \in \N_+\). 
    For any \(N \geq 3K T^2 / \beta\), any learner \(\alg\) of \(\point_N\) that only uses \(\point^K_N\) as its output set is \(\conc{\beta}\).
\end{lem}
\begin{proof}
    Set \(\xeq = 1\) and \(f_1 = \bI\bc{ \cdot = 2}\).
    Let \(\tau_0 = \br{\br{\xeq, f_1(\xeq)}, \ldots, \br{\xeq, f_1\br{\xeq}}}\) and
    \(\alg(\tau_0) = (\hat f_1, \ldots, \hat f_T)\).
    For \(t \in [T]\), let \(
        Q_t = \Set{x \in [N], \text{ such that } \Pr\br{\hat f_t(x) = 1} \geq \beta / T}\).
    We can write
    \begin{equation}
    \begin{aligned}
        &\sum_{x \in [N]} \Pr(\hat f_t(x) = 1) = \sum_{x \in [N]} \sum_{\substack{f \in \point^K_N \\ f(x) = 1}} \Pr(\hat f_t = f)
        \\
        &\quad = \sum_{f \in \point^K_N} \sum_{\substack{x \in [N] \\ f(x) = 1}} \Pr(\hat f_t = f) \leq K \sum_{f \in \point^K_N} \Pr(\hat f_t = f) = K,
    \end{aligned}
    \end{equation}
    which implies that \(\Card{Q_t} \leq K \lceil T / \beta \rceil\).
    Therefore, by union bound, as long as \(N \geq 3K T^2 / \beta\), there exists \(\xdif \in \left([N] \setminus \bigcup_{t=1}^T Q_t\right) \setminus \Set{1, 2}\).
    Applying union bound again, we obtain that 
    \begin{equation}
        \Pr(\exists t, \text{ such that } \hat f_t(\xdif) = 1) \leq \sum_{t=1}^T \Pr(\hat f_t(\xdif) = 1) \leq \beta.
    \end{equation}
    Thus, \(\alg\) is \(\conc{\beta}\) and the distinguishing tuple is \((f_1, \bI\bc{ \cdot = \xdif}, \xeq, \xdif)\).
\end{proof}

In~\Cref{lem:example-conc-learners} we assumed that \(K\) is a constant. Note that as long as \(K = o(N)\), by taking \(N\) large enough one can obtain the same result. We also remark that the~\Gls{soa} for \(\point_N\), which is an improper learner, uses \(\point_N^1\) as the output hypothesis class.~\citet{beimel2014bounds} also uses improper learners consisting of union of points for privately learning \(\point_N\) in the pure \Gls{dp}-\Gls{pac} model; their hypothesis class is a subset of the Multi-Point class.\footnote{ See Section 4.1 in~\citet{beimel2014bounds} to see the conditions that the hypothesis class needs to satisfy.} 

While their motivation was to explicitly construct a smaller hypothesis class that~\emph{approximates} \(\point_N\) well, on a more general note, it is also common to use larger hypothesis classes in computational learning theory for improper learning. This is particularly useful for benefits like computational efficiency~(\eg~3-DNF vs 3-CNF) and robustness~\citep{diakonikolas2019distribution}. The purpose of~\Cref{lem:example-conc-learners} is to show that natural improper learners using larger hypothesis classes for \(\point_N\), like~\Cref{defn:multi-point}, are also \(\mathrm{concentrated}\). Furthermore, as we show next, the only generic private online learner that we are aware of---~\Gls{dpsoa}, is also concentrated.
\begin{corollary}\label{lem:dpsoa-conc}
    \(\dpsoa\) for \(\point_N\) is \(\conc{\beta}\) for any \(N \geq 3 T^2 / \beta\).
\end{corollary}
\begin{proof}
    Output of \(\dpsoa\) is equal to the output of \(\soa\) algorithm on some input sequence. 
    Note that \(\soa\) for \(\point_N\) is always equal to either (i) an all-zero function, or (ii) a target function \(f^{\ast}\),
    which implies that it lies in \(\point^1_N\).
    Therefore, by~\cref{lem:example-conc-learners}, we obtain that \(\dpsoa\) is also \(\conc{\beta}\).
\end{proof}

\noindent\textbf{Independent work} In a concurrent and independent work,~\citet{cohen2024lower} established that learners that \emph{are not} concentrated must nevertheless suffer large expected regret for the specific case of \(\point_T\). They construct an adversary which strongly relies on the large size of the function class. In particular, generalizing for a smaller function classes, e.g.~\(\point_3\) remains an interesting open question. We obtain a partial progress in this direction, see~\Cref{sec:beyond_conc}.

\section{Discussion}
\label{sec:discussion}
\subsection{Connection to~\Gls{dp} under continual observation}\label{sec:continual}

\emph{Continual observation} under \Gls{dp}~\citep{dwork2010boosting} is the process of releasing statistics continuously on a stream of data while preserving~\Gls{dp} over the stream. One of simplest problems in this setting is~\Gls{dp}-Continual counting where a counter \(\cC: \bc{0,1}^T\to\bN_{+}^T\) is used.
We say \(\cC\) is deemed a \(\br{T,\alpha,\beta,\epsilon}\)-\Gls{dp} continual counter if \(\cC\) is \(\e\)-\Gls{dp} with respect to its input and with probability at least \(1-\beta\) satisfies
\begin{equation}\label{eq:continual-counter}
\max_{t\leq T} \abs{\cC\br{\tau}_t - \sum_{i\leq t}\tau_i}\leq \alpha.
\end{equation}
The proof of~\Cref{prop:continual} illustrates a straightforward method to convert~\Gls{dp} continual counters to~\Gls{dp} online learners for the \(\point_N\) hypothesis class, thereby transferring upper bounds from~\Gls{dp} continual counting to ~\Gls{dp} online learning. More precisely, the reduction results in a~\Gls{dp} online learning algorithm for \(\point_N\) with an additional \(\sqrt{N}\) factor to the privacy parameter and number of mistakes bounded by \(\alpha\). We believe this argument can also be extended to other finite classes by adjusting the mistake bound with an additional factor that depends on the size of the class.

 \begin{proposition}
 \label{prop:continual}
    For sufficiently small \(\epsilon,\beta\geq 0\) and any \(\alpha\geq 0\), let \(\cC\) be a \(\br{T,\alpha,\beta,\epsilon}\)-\Gls{dp} continual counter. Then, for any \(\delta>0\), an~\(\br{\epsilon',\delta}\)-\Gls{dp} online learner \(\cA\) for \(\point_N\) exists ensuring \(\bE\bs{\cM_{\alg}}\leq \alpha\) with \(\epsilon'=\epsilon\sqrt{3N\log\br{1 /\delta}}\) .
\end{proposition}

Several works~(see~\citet{chan2011private}) have identified counters with \(\alpha=\bigO{\br{\log T}^{1.5}/\epsilon}\) which immediately implies a mistake bound that also scales as \(\poly\log \br{T}/\epsilon\) using~\Cref{prop:continual}. Ignoring\footnote{We do not optimise this as this dependence is not the focus of this work, for this argument consider \(N=\bigO{1}\). However, we believe it can be reduced using a~\Gls{dp} continual algorithm for~\textrm{MAXSUM}; see~\citet{jain2023price}} the dependence on~\(N\), this almost matches the lower bound proposed in~\Cref{thm:main-finite}. On the other hand,~\citet{dwork2010differential} have shown a lower bound of \(\Omega\br{\frac{\log T}{\epsilon}}\) for \(\alpha\) for any pure~\Gls{dp}-continual counter. However it is not clear how to convert the lower bound for~\Gls{dp}-continual learning to~\Gls{dp} online learning. This is because 1)~\Cref{eq:continual-counter} asks for a uniformly~(across \(t\)) accurate counter whereas the mistake bound in~\Cref{eq:mistake-bound} is a global measure and 2) this lower bound is for pure~\Gls{dp} whereas our setting is approximate~\Gls{dp}. We leave this question for the future work.

\subsection{Beyond concentration assumption}
\label{sec:beyond_conc}
In~\Cref{sec:finite_horizon_result}, we analyzed a class of algorithms, namely concentrated algorithms, containing the only existing~\Gls{dp} online learning algorithm,~\Gls{dpsoa}, 
and provided a matching logarithmic in \(T\) lower bound.
A natural question is whether this lower bound can be extended to \textit{any} learner. While we do not provide a general answer, we show that our construction can be made more general. Here, we analyze a different class of algorithms, called \emph{firing algorithms}.

For simplicity of exposition, consider the hypothesis class \(\point_3\), 
and let \(f^{\ast} \in \Set{f^{(1)}, f^{(2)}}\).
Let \(x^\ast\) be such that \(f^\ast = f^{(x^\ast)}\) and assume that adversary only considers sequences consisting of \((x^\ast, 1)\) and \((3, 0)\).
Let \(\cD\) be a distribution on \(\Set{\text{all-zero function}, f^{(1)}, f^{(2)}}\).
At each point \(t\), a firing algorithm \(\alg\) computes \(p_t = p_t(\#\text{mistakes up until }t) \in [0, 1]\) with \(p_t(0) = 0\), and samples a \(\xi_t \sim \mathrm{Bern}(p_t)\). If \(\xi_t = 1\), the algorithm commits to the correct hypothesis henceforth; otherwise, it outputs \(f_t \sim \cD\). Note that if the adversary introduces \((3, 0)\) at step \(t\), \(\alg\) is guaranteed not to err, and a single mistake suffices to identify (non-privately) the correct hypothesis.

When \(\cD\) has support on only one of \(\Set{\text{all-zero function}, f^{(1)}, f^{(2)}}\), it yields a \(\conc{0}\) algorithm.
Furthermore, the continual observation algorithm can also be viewed as a firing algorithm with a proper choice of \(\cD\). We analyze the opposite to the concentrated algorithms, in particular when \(\cD = \mathrm{Unif}\Set{f^{(1)}, f^{(2)}}\). We call such learners \emph{uniform firing algorithms} and we also obtain a logarithmic in \(T\) lower bound for them.

\begin{proposition}\label{prop:firing-lb}
    Let \(\alg\) be an \((\e, \delta)\)-\Gls{dp} uniform firing algorithm for class \(\point_3\).
    Then, if \(\log T = \bigO{1 / \delta}\), there exists an adversary, such that \(\E \bs{M_{\alg}} = \Omega(\log T)\). 
\end{proposition}

Proof is provided in~\Cref{app:uniform_lb} and is similar to the proof of~\Cref{thm:main-finite}, but requires a more delicate construction. 

\subsection{Pure differentially private online learners}

While we have so far focused on Approximate~\Gls{dp} with \(\delta > 0\), in this section we briefly discuss Online learning under pure~\Gls{dp}. 
Note that~\cref{lem:example-conc-learners} and~\cref{thm:main-finite} immediately imply the following lower bound on pure differentially private online learners for \(\point_N\).
\begin{corollary}\label{corr:pure-dp-lower}
    Let \(\e > 0\), \(\beta > 0\), \(K, T \in \N_+\) and \(N \geq 3KT^2 / \beta \).
    For any \(\e\)-\Gls{dp} learner \(\alg\) which uses only \(\point^K_N\) as its output set, 
    there exists an adversary \(\cB\), such that
    \begin{equation}
        \E \bs{M_{\alg}} = \Omega\br{\min\br{\frac{\log T / \beta}{\e}, T}}.
    \end{equation}
\end{corollary}
\begin{proof}
    \cref{lem:example-conc-learners} implies that \(\alg\) must be \(\conc{\beta}\). Furthermore, since \(\alg\) is \(\e\)-\Gls{dp}, it is also \((\e, \beta)\)-\Gls{dp},
    and, thus, the existence of adversary with large regret follows from~\cref{thm:main-finite}.
\end{proof}

For \((\e, \delta)\)-\Gls{dp} online algorithms, there exists an upper bound provided by~\citet{golowich2021littlestone}. However, to the best of our knowledge, not much is known about \(\e\)-\Gls{dp} online algorithms. One way to obtain such algorithm is by leveraging existing results from the continual observation literature as done in~\Cref{prop:continual}. Under the same assumptions as in~\Cref{prop:continual}, using basic composition instead of advanced composition in the last step results in an \(\epsilon'\) scaling as \(N\epsilon\). However, this only works for~\(\point_N\) and not for more general classes. For arbitrary finite hypothesis classes, it is possible to use a~\Gls{dp} continual counter, similar to above to obtain a mistake bound that also scales with the size of the class. We also remark that the \(\mathsf{AboveThreshold}\) algorithm could be used to design learners for certain function classes. However, the question of whether generic learners can be designed remains open. Another interesting open question is whether the dependence on the size of the class \eg \(N\) for \(\point_N\), is necessary.

\subsection{Open problems}

This work relies on assumptions about properties~(Concentrated Assumption in~\Cref{defn:beta-conc} or uniform firing Assumption in~\Cref{sec:beyond_conc}) of the learning algorithms to show a lower bound for any hypothesis class. On the other hand,~\citet{cohen2024lower} do not need any assumption on their algorithm but their result only holds for large \(\point\) classes~(and it cannot immediately be transferred to small \(\point\) classes, e.g. \(\point_3\)). To fully settle the problem of lower bounds for online~\Gls{dp} learners, the limitations of both this work and that of~\citet{cohen2024lower} need to be addressed. We propose the following conjecture which we believe is true.
\begin{conjecture}[Lower bound for Approximate DP]
\label{conj:approx-dp-lower}
For any \(\e, \delta > 0\) and any \((\e, \delta)\)-\Gls{dp} learner \(\cA\) of \(\point_3\), there exists an adversary \(\cB\), such that \(\E \bs{M_{\cA}} = \Omega\br{\min\br{\frac{\log T / \delta}{\e}, 1 / \delta}}\).
\end{conjecture}

A straightforward implication of our main result is that the \(\log T\) lower bound also holds for pure~\Gls{dp} online learners under the same assumptions on the algorithm. In~\Cref{sec:continual}, we showed a generic reduction from~\Gls{dp} continual counters to~\Gls{dp} online learners for \(\point_N\) with regret \(\poly\log \br{T}\). This reduction can be extended to  pure~\Gls{dp} online learner by using basic composition in~\Cref{prop:continual}, with an additional cost of \(\sqrt{N}\). In~\Cref{conj-puredp-ub-lb}, we raise the question whether for small \(N\), the dependence on \(T\) can be lowered to \(\log T\) and shown to be tight.



\begin{conjecture}[Upper and lower bounds for Pure DP]\label{conj-puredp-ub-lb}
    For any \(\e > 0\),\begin{enumerate}
        \item \emph{Upper bound: } 
        there exists  \(\e\)-\Gls{dp} learner of \(\point_3\), s.t. for any adversary, \(\E \bs{M} = O\br{\frac{\log T}{\e}}\).
        \item \emph{Lower bound:} for any \(\e\)-\Gls{dp} learner of \(\point_3\), there exists an adversary, s.t. \(\E \bs{M} = \Omega\br{\frac{\log T}{\e}}\).
    \end{enumerate} 
\end{conjecture}

Note that the lower bound part of~\Cref{conj-puredp-ub-lb} follows from~\Cref{conj:approx-dp-lower}, but can also be viewed through connection to the continual observation model. In the latter regime, an \(\Omega(\log T)\) lower bound was shown in~\citet{dwork2014algorithmic}.

\section{Acknowledgement}
DD is supported by ETH AI Center doctoral fellowship and ETH Foundations of Data Science initiative. AS and DD thank Yevgeny Seldin for hosting them in University of Copenhagen which helped in the discussion of this work and Amir Yehudayoff for helpful discussions. DD also thanks Afonso Bandeira and Petar Nizić-Nikolac for fruitful discussions during the course of the project. In addition, authors would like to thank Xiaoyu Wang for pointing out a mistake in the proof of Theorem 1 in the earlier version which has now been fixed.

\bibliographystyle{abbrvnat}
\bibliography{biblio}

\appendix
\crefalias{section}{appendix} 

\clearpage

\section{Proof of~\Cref{prop:continual}}
\begin{proof} The proof for~\Cref{prop:continual} simply follows from first instantiating \(\abs{\cH}\) separate \(\br{T,\alpha,\beta,\epsilon'}\)-\Gls{dp} continual counters, one for each \(h\in\cH\). 

Then we construct the learning algorithm as follows: at each step \(t\), the learner outputs the hypothesis that predicts \(0\) uniformly. This is possible as the learning algorithm is improper. Then the adversary provides the learner with \(\br{x_t, f^*\br{x_t}}\). The algorithm then privately updates the counter \(\cC_{h}\) with \(b_h = \bI\bc{f^*\br{x_t} = h\br{x_t}}\wedge f^*\br{x_t}\). Then, as soon as any counter's value~(say \(\cC_{h^*}\)) surpasses \(\alpha\), only predict using \(h^*\).

Note that this algorithm only incurs a mistake at step \(t\) if \(b_h=1\) at step \(t\). However, every time \(b_h=1\), the counter for \(f^*\) gets updated by one. Thus, after \(\alpha\) updates it can be guaranteed with probability \(1-\beta\) that \(h^* = f^*\). To get the final expected mistake bound, set \(\beta = \alpha\) and this completes the proof.

The privacy proof follows from applying strong composition to the different counters.
\end{proof}
\section{Lower bound for a non-concentrated algorithm}
\label{app:uniform_lb}
Consider the problem of learning function class \(\point_3\), 
such that \(f^{\ast} \in \Set{f^{(1)}, f^{(2)}}\).
Let \(x^\ast\) be such that \(f^\ast = f^{(x^\ast)}\) and assume that the adversary only considers sequences consisting of \((x^\ast, 1)\) and \((3, 0)\). In this section, we show a lower bound for one special type of non-concentrated algorithm:~\emph{firing algorithms}. 

At each round \(t\), the algorithm computes \(p_t = p_t\br{\#\text{mistakes up until }t} \in [0, 1]\) with \(p_t(0) = 0\), and samples \(\xi_t \sim \mathrm{Bern}(p_t)\). If \(\xi_t = 1\), from this point on, the  algorithm always outputs the correct hypothesis. In this case, we say that \(\alg\)  has `fired' at step \(t\). Otherwise, it outputs \(f_t \sim \mathrm{Unif}\{f^{(1)}, f^{(2)}\}\). We call such learners \emph{uniform firing algorithms}.
Note that \(\alg\) never errs on points \((3, 0)\), since both \(f^{(1)}\) and \(f^{(2)}\) are equal to zero at \(x = 3\).
Essentially, we analyze a distinguishing tuple \((f^{(1)}, f^{(2)}, 3, x^{\ast})\).

\begin{proof}[of~\Cref{prop:firing-lb}]
WLOG, we assume \(\delta > 1/T^2\).
For the elements of the input sequence, we denote \(0 \coloneqq (3, 0)\) and \(1 \coloneqq (x^\ast, 1)\).
Similar to the proof of~\cref{thm:main-finite}, we proceed by creating sequences \(\tau_i\), with \(\tau_0 = (0, \ldots, 0)\) and \(\tau_i\) containing exactly \(i\) copies of \(1\).

Assume that there is a scalar sequence \((r_i)\) and a sequence of subset of events \((G_i)\), such that
\begin{enumerate}
    \item \(\Pr(G_i) \geq 1 - \bigO{i \delta} \geq 1/2\),
    \item \(\Pr(\alg(\tau_i)\text{ outputs only }f^\ast \text{ starting from }r_i \mid G_i) = \bigO{\delta}\).
\end{enumerate}
This implies that we can upper bound probability that \(\alg\) outputs \(f^*\) for any \(t \leq r_i\) as follows (for a subset of events \(A\), we denote its complement by \(\overline{A}\) and \(\Pr_i (A) \coloneqq \Pr(A \mid G_i)\)):
\begin{equation}\label{eq:marginal}
\begin{aligned}
    \Pr(\alg(\tau_i)_t = f^\ast) & = \bP_i(\alg(\tau_i)_t = f^\ast)\Pr(G_i) + \Pr(\alg(\tau_i)_t = f^\ast \mid \overline{G_i})\Pr(\overline{G_i}) \\
    & \leq \bP_i(\alg(\tau_i)_t = f^\ast) + O(i \delta) \\
    &= O(i\delta) + \bP_i(\alg(\tau_i)_t = f^\ast \mid\alg \text{ `fired' before }r_i)\bP_i(\alg(\tau_i) \text{ `fired' before }r_i) \\
    &\qquad + \bP_i(\alg(\tau_i)_t = f^\ast \mid \alg(\tau_i) \text{ did not `fire' before }r_i)\bP_i(\alg(\tau_i) \text{ did not `fire' before }r_i)  \\
    &\leq \bP_i(\alg(\tau_i) \text{ `fired' before }r_i) + \frac{1}{2} + O(i\delta) \\
    &\leq \bP_i(\alg(\tau_i)\text{ outputs only~}f^\ast\text{ starting from }r_i) + \frac{1}{2} + O(i\delta)\\
    &= \frac{1}{2} + O(i \delta).
\end{aligned}
\end{equation}
For \(i = O(1 / \delta)\), we obtain that \(\Pr(\alg(\tau_i)_t \neq f^*) = \Omega(1)\).
If \(\tau_i\) contains exactly \(i\) points \((x^\ast, 1)\) on the prefix \([1, r_i]\), then, by picking \(i = \Omega(\log T)\) (recall that \(\log T = O(1 / \delta))\) and summing the latter probabilities, we obtain that \(\E \bs{M_\alg} = \Omega(\log T)\).

We now construct sequences \((r_i), (G_i), (\tau_i)\), also including a scalar sequence \((l_i)\), such that \([l_i, r_i]\) are nested segments and \((r_i), (G_i)\) satisfy the aforementioned requirements. We proceed by induction. First, define \(l_0 = 1\), \(r_0 = T/2\), \(\tau_0 = (0, \ldots, 0)\), and \(G_0 = \Omega\), i.e., so that \(\Pr(G_0) = 1\). Note that \(\Pr(\alg(\tau_0)\text{ outputs only }f^\ast \text{ starting from }r_0) = 2^{-T/2} \leq \delta\).
Next, assume that we already defined \(l_i, r_i, G_i\).
We divide a segment \([l_i, r_i]\) into three parts \([l_i, lm_i], [lm_i, rm_i], [rm_i, r_i]\), with \(lm_i = l_i + (r_i - l_i) / 3\) and \(rm_i = l_i + 2(r_i - l_i) / 3\). We have the following decomposition:
\begin{equation}
    \begin{aligned}
        &\bP_i(\alg(\tau_i)\text{ outputs only }f^\ast \text{ starting from }r_i) \\
        &\quad = \bP_i(\alg(\tau_i)\text{ outputs only }f^\ast \text{ starting from }r_i \textbf{ and } \text{ outputs only }f^\ast \text{ starting from }[lm_i, rm_i]) \\
        &\qquad+\bP_i(\alg(\tau_i)\text{ outputs only }f^\ast \text{ starting from }r_i \textbf{ and } \text{does not only output }f^\ast \text{ on }[lm_i, rm_i]).
    \end{aligned}
\end{equation}

As long as \(rm_i - lm_i \geq \sqrt{T}\), we obtain
\begin{equation}
    \begin{aligned}
        &\bP_i(\alg(\tau_i)\text{ outputs only }f^\ast \text{ on }[lm_i, rm_i]) \\
        &\quad = \bP_i(\alg(\tau_i)\text{ fires before }lm_i) + \sum_{t=1}^{rm_i - lm_i}2^{-t} \bP_i(\alg(\tau_i)\text{ fires on }lm_i+t) + 2^{-(lm_i - rm_i)}  \\
        &\quad \leq 2\bP_i(\alg(\tau_i)\text{ outputs only }f^\ast \text{ starting from }r_i) + 2^{-\sqrt{T}}.
    \end{aligned}
\end{equation}

We let \(\tau_{i+1} = \tau_i\), except that \((\tau_{i+1})_{l_i} = 1\). Then, using DP property we have that 
\begin{equation}
\begin{aligned}
    &\bP_i(\alg(\tau_{i+1})\text{ outputs only }f^\ast \text{ starting from }r_i) \\
    &\quad \leq \exp(\varepsilon) \bP_i(\alg(\tau_{i})\text{ outputs only }f^\ast \text{ starting from }r_i) + \delta / \Pr(G_i) \\
    &\quad \leq \exp(\varepsilon) \bP_i(\alg(\tau_{i})\text{ outputs only }f^\ast \text{ starting from }r_i) + 2\delta.
\end{aligned}
\end{equation}
Therefore, either
\begin{equation}
\begin{aligned}
&\Pr(\alg(\tau_{i+1})\text{ outputs only }f^\ast \text{ starting from }r_i \textbf{ and } \text{ outputs only }f^\ast \text{ on }[lm_i, rm_i])\\
&\quad \leq \frac{1}{2}(\exp(\varepsilon) \Pr(\alg(\tau_{i})\text{ outputs only }f^\ast \text{ starting from }r_i) + \delta),
\end{aligned}
\end{equation}
or 
\begin{equation}
\begin{aligned}
&\Pr(\alg(\tau_{i+1}\text{ outputs only }f^\ast \text{ starting from }r_i \textbf{ and } \text{does not only output }f^\ast \text{ on }[lm_i, rm_i])\\
&\quad \leq \frac{1}{2}(\exp(\varepsilon) \Pr(\alg(\tau_{i})\text{ outputs only }f^\ast \text{ starting from }r_i) + \delta).
\end{aligned}
\end{equation}
In the former case, we can reiterate with \(l_{i+1} = l_i + 1\), \(r_{i+1} = lm_i\), and \(G_{i+1} = G_i\), using that 
\begin{equation}
    \Pr(\alg(\tau_{i+1})\text{ outputs only }f^\ast \text{ starting from }lm_i) \leq \Pr(\alg(\tau_{i+1})\text{ outputs all ones on }[lm_i, rm_i]).
\end{equation}
In the latter case, we set \(G_{i+1} = G_i \cap \Set{\alg(\tau_{i+1})\text{ does not only output }f^\ast \text{ on }[lm_i, rm_i]}\).
Note that \(\Pr(G_{i+1}) = \Pr(G_i) \bP_i (\alg(\tau_{i+1})\text{ does not only output }f^\ast \text{ on }[lm_i, rm_i]) \geq 1 - O((i+1)\delta)\).
We reiterate with \(l_{i+1} = rm_i\) and \(r_{i+1} = r_i\). 

Altogether, as long as \(rm_i - lm_i \geq \sqrt{T}\), we can always maintain that 
\begin{enumerate}
    \item \(\Pr(G_{i+1}) \geq 1 - O(i \delta)\),
    \item \(\Pr(\alg(\tau_i)\text{ outputs only }f^\ast \text{ starting from }r_i \mid G_i) = O(\delta).\)
\end{enumerate}
We can therefore continue this process at least \(\Omega(\log T)\) times, which implies a large number of mistakes as discussed before.
\end{proof}

Proof of~\Cref{prop:firing-lb} resembles the proof of~\cref{thm:main-finite} with several important differences.
The main difficulty comes from the fact that by just looking at the output of the algorithm, it is impossible to say with certainty, whether it `fired' or not.
And we need to do this, in order to bound the marginal probability of outputting \(f^\ast\) (which is \(1/2\) if \(\alg\) did not fire and \(1\) if it fired), see~\cref{eq:marginal}.

To overcome this issue, instead of splitting the segment into two parts and reiterating in either left or right,
we split into three parts and reiterate in the first or the third.
Second segment plays a special role, to 'detect' whether algorithm fired or not.
From the properties of the algorithm, if the full segment is equal to all \(f^\ast\),
we can confidently say that algorithm fired before this segment.
Event \(G_i\) corresponds to a 'good' event, meaning that \(\alg\) did not fire and \emph{did not look like it fired}, i.e., in the middle segment there exist one output not equal to \(f^\ast\).

Finally, we would like to highlight that there is nothing specific about the choice \(\cD = \mathrm{Unif}(f^{(1)}, f^{(2)})\).
The main property we needed for the proof is that \[\Pr(\alg \text{ outputs all } f^\ast \text{ on a specific subsequence of large enough length }\mid \alg \text{ did not fire }) \leq \delta.\]
Therefore, we expect that similar technique works for other choices of \(\cD\), beyond a uniform distribution.

 \section{Extending finite lower bound to other values of \(\e\)}
\label{app:all-eps}
\paragraph{Smaller \(\e\)}
Let \(d \in \N_+\) and \(\e = \e_0 / d\), where \(\e_0 = \log(3/2)\).
WLOG we assume that \(T = d (2^k - 1)\). 
We repeat the same construction as in the proof of~\cref{thm:main-finite} for the case \(\e_0\),
but now instead of inserting only one point at each step, we insert \(d\) points.
Note that this way sequences \(\inpseq^{(i+1)}\) and \(\inpseq^{(i)}\) are \(d\) points away from each other, therefore, by applying \Gls{dp} property of \(\alg\) \(d\) times, 
we obtain the following equivalent version of~\cref{eq:q-i-plus-one}:
\begin{equation}
    \min\br{\Pr\bs{Q_1(\inpseq^{(i+1)})}, \Pr\bs{Q_2(\inpseq^{(i+1)})}}
    \leq \frac{1}{2} \br{\exp(\varepsilon_0) \Pr\bs{Q(\inpseq^{(i)})} + A\delta},
\end{equation}
where \(A = 1 + \exp(\e_0 / d) + \exp(2\e_0 / d) + \ldots + \exp((d - 1) \e_0 / d) \leq d \exp(\eps_0).\)
Therefore, multiplying \(\delta\) by a factor \(d \exp(\eps_0)\), we recover the previous setting. 
Since only \(\log 1 / \delta\) appears in the final bound, and since \(\eps_0\) is a constant, we only suffer an extra \(\log d \sim \log 1 / \e\) term,
which is subleading.

\paragraph{Larger \(\e\)}
When \(\e > \e_0\), instead of dividing the sequence in two parts, we will divide it into \(s\) parts, for some integer \(s \geq 3\), obtaining
\begin{equation}
    \min\br{\Pr\bs{Q_1(\inpseq^{(i+1)})}, \ldots, \Pr\bs{Q_s(\inpseq^{(i+1)})}}
    \leq \frac{1}{s} \br{\exp(\e) \Pr\bs{Q(\inpseq^{(i)})} + \delta}.
\end{equation}
We need to ensure that \(\exp(\e) / s \leq 3/4\), in order for the proof to go through, which follows if \(\log s = O(\e)\). 
Finally, note that instead of repeating process \(\Omega(\log T)\) times, we repeat \(\Omega\br{\frac{\log T}{\log s}} = \Omega\br{\frac{\log T}{\e}}\) times,
which matches the required mistake bound.

\end{document}